\documentclass[a4paper,UKenglish,cleveref, autoref, thm-restate]{lipics-v2021}

\title{A Machine Learning Theory Perspective on Strategic Litigation} 

\author{Melissa Dutz}{Toyota Technological Institute at Chicago, USA}{melissa@ttic.edu}{https://orcid.org/0009-0008-0971-8285}{Melissa Dutz was supported in part by an NSF CSGrad4US fellowship under Grant No. 2313998. Any opinions, findings, and conclusions or recommendations expressed in this material are those of the author(s) and do not necessarily reflect the views of the U.S. National Science Foundation.}

\author{Han Shao\footnote{Work mostly done while at Toyota Technological Institute at Chicago}}{University of Maryland, USA \and \url{https://www.cs.umd.edu/~hanshao/}}{hanshao@umd.edu}{https://orcid.org/0009-0005-9206-1357}{Han Shao was supported in part by Harvard CMSA.}

\author{Avrim Blum}{Toyota Technological Institute at Chicago, USA \and \url{https://home.ttic.edu/~avrim/}}{avrim@ttic.edu}{https://orcid.org/0000-0003-2450-5102}{Avrim Blum was supported in part by the National Science Foundation under grant CCF-2212968, and by ECCS-2216899. Any opinions, findings, and conclusions or recommendations expressed in this material are those of the author(s) and do not necessarily reflect the views of the U.S. National Science Foundation.}

\author{Aloni Cohen}{The University of Chicago, USA \and \url{https://aloni.net}}{aloni@uchicago.edu}{https://orcid.org/0000-0002-3492-2447}{}

\authorrunning{M. Dutz, H. Shao, A. Blum, and A. Cohen} 

\Copyright{Melissa Dutz, Han Shao, Avrim Blum and Aloni Cohen} 



\category{} 

\relatedversion{} 

\nolinenumbers 
\hideLIPIcs

\usepackage{float}
\usepackage{dsfont}
\usepackage{algorithm}
\usepackage{algorithmic}
\usepackage{bbm}
\usepackage{mdframed}

\usepackage[skip=2pt, font=small, labelfont=bf]{caption}
\usepackage{placeins}   
\usepackage[font=small,aboveskip=2pt,belowskip=2pt]{caption}

\begin{document}

\maketitle

\begin{abstract}
\textit{Strategic litigation} involves bringing a case to court with the goal of having an impact beyond resolving the particular dispute at hand. In a common law system, one way a case may have far-reaching impact is by establishing new legal precedent that later courts must follow. In this paper, we explore strategic litigation from the perspective of machine learning theory. We consider an abstract model of a common law legal system where a lower court decides new cases by applying a decision rule learned from a higher court's past rulings. In this model, we explore the power of a \textit{strategic litigator}, who strategically brings cases to the higher court to influence the decision rule applied by the lower court in future cases. We explore questions including: What impact can a strategic litigator have? Which cases should a strategic litigator bring to court? Does it ever make sense for a strategic litigator to bring a case when they are sure the court will rule against them? We show that this strategic case selection problem has interesting structure, with even simple settings exhibiting counterintuitive phenomena. When cases are represented by points in one dimension and the lower court’s learning algorithm is nearest neighbor, or as points in $d$ dimensions and the lower court’s learning algorithm is a support vector machine, we characterize the set of inducible decision rules and develop algorithms for selecting an optimal set of cases to bring to the higher court given the strategic litigator’s objectives.
\end{abstract}

\section{Introduction}
Strategic litigation (also called impact litigation) involves bringing a legal case to court with the goal of creating some broader societal impact outside of resolving the case itself: for example, shaping future law. In a common law system, a decision on one case can establish precedent that future cases must follow. Proponents of a particular legal doctrine might therefore want to carefully select a case likely to be ruled on favorably, so that similar but perhaps more ambiguous cases will be decided favorably as well. A famous example of strategic litigation is \textit{Brown v. Board of Education}; although this case concerned school desegregation specifically, it was crafted to undermine the ``separate but equal'' doctrine of racial discrimination much more broadly. 

The United States has a hierarchical court system in which higher courts' rulings establish binding precedent that lower courts must follow. Lower courts must somehow determine how precedent applies to new cases. This process can be viewed as the lower courts learning a decision rule from precedent: a way to map information about a new case to a ruling based on precedent. This perspective motivates studying the application of precedent through the lens of learning theory, the study of learning functions from data.  

In this paper, we make an initial foray into studying strategic litigation from a machine learning theory perspective. We introduce a highly stylized, abstract model of a common-law legal system with three actors: a high court, a lower court, and a strategic litigator. The lower court decides new cases by applying a decision rule learned from a higher court’s past rulings. The strategic litigator strategically brings cases to the higher court to influence this learned decision rule, thereby affecting future cases. We explore questions including: What impact can a strategic litigator have? Which cases should a strategic litigator bring to court? Does it ever make sense for a strategic litigator to bring a case when they know the court will rule against them? Under two variants of our model, we characterize the set of decision rules a strategic litigator can induce, and provide algorithmic results for selecting an optimal set of cases to bring to court given the strategic litigator's objectives. Of course, studying such a simplistic model cannot necessarily tell us anything about the US court system in all its complexity. However, we hope it can still offer some insight into the dynamics of strategic litigation. 

This work has connections to other topics in machine learning theory, including teaching dimension (\cite{teaching_dimension}, \cite{teaching_linear_learner}) and clean-label data attacks (\cite{poison_frogs}, \cite{blum2021robust}). Previous works have also modeled legal decision making as a classifier (e.g., \cite{nn_1974}, \cite{branting2021scalable}, \cite{hartline_et_al}) but to our knowledge this is the first exploration of strategic impact litigation in a machine learning setting. It will be helpful to present our model before discussing these connections in more detail, so we will revisit them in Section \ref{sec:prev_work}. 

\section{Setting}
In this section, we describe our basic abstract model of the legal system, specify the instantiations of our model that we will use as case studies throughout this paper, and give an example illustrating strategic litigation under our model. 

\subsection{Abstract Model of the Legal System}
We introduce our basic abstract model of cases, courts and strategic litigators below. Our model does not aim to realistically represent a real-world court system. Instead, our goal is to describe a simple model of a common-law legal system in which strategic litigation is possible. This allows us to study the dynamics of strategic litigation with minimal confounding complexity. We will extend our model in Section \ref{sec:extensions} to consider overturning precedent.

\textbf{Modeling Case Information:} The key facts of cases are referred to as \textit{case fact patterns}. We represent case fact patterns as data points in some instance space. 

\textbf{Modeling the Courts:} We consider a simple two-level court system comprised of a high court and lower courts. Courts act as binary classifiers on case fact patterns, with positive and negative labels corresponding to decisions for and against the plaintiff, respectively. The high court and the low court decide cases differently.
The high court decides cases according to some binary labeling function $f^*$. The lower court learns a binary labeling function $f$ from precedent (those cases already decided/labeled by the higher court) using a deterministic learning algorithm $\mathcal{A}$. 
Denoting the set of historical labeled data $S_h = \{(x_1, f^*(x_1)), (x_2, f^*(x_2)), ..., (x_n, f^*(x_n))\}$, the classification function learned by the lower courts is $f \gets \mathcal{A}(S_h)$. Note that each time a new case is decided by the high court, the set of precedent $S_h$ is updated and the function learned by the lower courts is updated accordingly. We assume that $\mathcal{A}$ always produces a function consistent with all precedent (e.g. $f(x) = y ~ \forall (x,y) \in S_h$). Finally, we assume that $f^*$ is realizable, meaning that $f^*$ is in the hypothesis class $\mathcal{H}$ of functions that $\mathcal{A}$ returns.

\textbf{Modeling a Strategic Litigator:} We assume a strategic litigator has some goal function $g \in \mathcal{H}$ which describes how they would like cases to be decided, and we assume they know the high court's labeling function $f^*$ and the learning algorithm $\mathcal{A}$ used by the lower courts. Given $g$, the strategic litigator would like to bring some cases to be decided by the high court, such that the precedent they create influences the lower courts to rule on new cases according to $g$. 

More precisely, the strategic litigator would like to select a set of case fact patterns to be labeled by $f^*$ (call this labeled set $S_l$), such that the function learned by the lower courts based on the union of the existing labeled set $S_h$ and the added set $S_l$ is close to their goal function $g$. Note that in general, it may not be possible to force $\mathcal{A}$ to output exactly $g$. Therefore, the general goal of the strategic litigator is to select a set of case fact patterns such that the function $f$ output by $\mathcal{A}$ minimizes classification error (as defined by $g$) with respect to some distribution $\mathcal{D}$ of case fact patterns that the strategic litigator cares about. Formally:
\begin{multline}
  \text{Given $f^*$, $g$, $S_h$, find $S_l$ minimizing } 
  \mathds{E}_{x \sim \mathcal{D}}\!\Big[\mathds{1}\{g(x) \neq f(x)\}\Big] 
  \text{ where } f \gets \mathcal{A}(S_h \cup S_l).
\end{multline}
We primarily focus on the setting where the strategic litigator can select cases from some finite pool of case fact patterns, $P$. This captures the fact that only a finite set of existing cases can be brought to court (one cannot litigate a hypothetical case). However, we will also discuss the case where the strategic litigator can bring \textit{any} point in the instance space when this enables more interesting results. 

\subsection{Case Studies}
Throughout this paper, we will focus on two particular instantiations of our model:

\begin{enumerate}
    \item \textbf{1-Dimensional Nearest Neighbor Setting}: Case fact patterns are represented by points in 1-dimensional space, the hypothesis class $\mathcal{H}$ consists of all Boolean functions over $\mathds{R}$ that make a finite number of alternations between positive and negative labels, 
    and the lower court's learning algorithm $\mathcal{A}$ is a \textit{nearest neighbor classifier}, which labels a new point the same way as the closest point in its training data. 
    \item \textbf{$d$-Dimensional Linear Separator Setting}: Case fact patterns are represented by points in $d$-dimensional space, the hypothesis class $\mathcal{H}$ consists of linear separators, and the lower court's learning algorithm $\mathcal{A}$ is a \textit{support vector machine} (SVM).  Recall that SVM finds a maximum-margin linear separator (a separator which maximizes the distance between itself and the closest point from each class). The position of the separating hyperplane which SVM learns is completely determined by training points which lie on the margin (e.g., the points closest to the separator), which are called support vectors. 
\end{enumerate}

We study nearest neighbor classifiers and maximum-margin linear separators because they are simple and natural heuristics for decision-making based on historical data. For example, nearest neighbor captures the intuition that similar cases should be decided similarly. A linear separator is natural as it is the simplest way to divide a $d$-dimensional space into two classes, such as winning and losing cases. By maximizing the distance between the boundary and the most ambiguous `edge cases' in each class, maximum-margin linear separators are intuitively robust among possible linear separators.

Figure \ref{fig:example_sl} shows a simple example of strategic litigation in our 1-Dimensional Nearest Neighbor setting.

\begin{figure}[H]
    \centering
    \includegraphics[width=.6\linewidth]{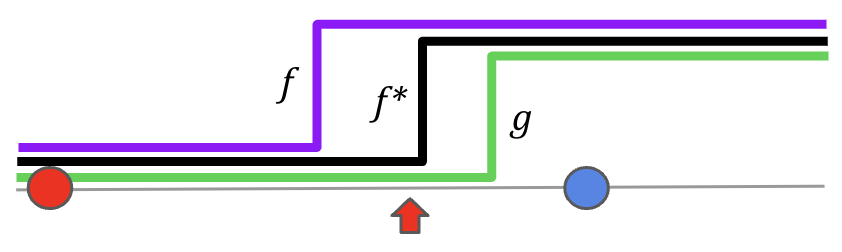}
    \caption{$f^*$ is the labeling function used by the high court, $f$ is the output of Nearest Neighbor given the two historical points labeled by $f^*$ (the red point is negative, the blue point is positive), and $g$ is the strategic litigator's goal function. To make Nearest Neighbor output $g$, the strategic litigator can bring a point where indicated by the red arrow, which will be labeled negative by $f^*$.}
    \label{fig:example_sl}
\end{figure}
\vspace{-1em} 

\section{Related Work}\label{sec:prev_work}

As mentioned earlier, our work is related to the concept of the \textit{teaching dimension} in learning theory. Introduced by \cite{teaching_dimension}, the teaching dimension measures the complexity of \textit{teaching}. Suppose a teacher wants to select the most helpful examples to efficiently teach any student a concept in a given concept class; the teaching dimension measures the minimum number of examples required. More formally, the teaching dimension is the minimum number of examples necessary to teach \textit{any} consistent learning algorithm a target function $f_t$ from a given hypothesis class $\mathcal{H}$ (e.g., it is the minimum number of examples such that $f_t$ is the only remaining consistent function in $\mathcal{H}$). In our setting, the strategic litigator faces a similar problem to the teacher: they would like to find a small set of examples to ``teach'' $g$ to the lower court's learning algorithm $\mathcal{A}$. This differs from the classic teaching dimension setting in two main ways. First, in the classic teaching setting, examples are labeled by the teacher's target function $f_t$. In our case, examples are not labeled by the strategic litigator's target function $g$ -- they are labeled by $f^*$. Second, the classic teaching problem requires a set of examples which will teach $\textit{any}$ consistent learner, while in our model the strategic litigator can focus on teaching a particular, known learner ($\mathcal{A}$). More recent work (\cite{teaching_linear_learner}) has studied the teaching dimension given a known learner, including SVM. In particular, \cite{teaching_linear_learner} show that the number of examples required to teach any particular linear separator to SVM is 2. This corresponds to our setting when $\mathcal{A}$ is SVM in the special case where $f^* = g$, there is no historical data, and the strategic litigator can select any case fact pattern in the instance space.

This work is also connected to work on clean label data poisoning attacks (see \cite{poison_frogs}, \cite{blum2021robust}). In a clean-label data poisoning attack, an attacker seeks to cause the learning algorithm to make a mistake on a target test instance by injecting strategically chosen, correctly-labeled examples into the algorithm's training set. In our work, the strategic litigator similarly seeks to manipulate the output of a learning algorithm ($\mathcal{A}$) by injecting correctly labeled examples (case fact patterns labeled by $f^*$) into its training set. The clean-label data poisoning model would correspond to a variant of our setting in which the strategic litigator's goal is to cause the lower courts to rule favorably on a particular future case, rather than trying to bring the function $\mathcal{A}$ learns close to some goal function $g$ overall.

A number of previous works such as \cite{nn_1974}, \cite{branting2021scalable} and \cite{hartline_et_al} have also modeled legal decisions as the output of some learning algorithm; the key novelty of our work lies in exploring strategic impact litigation under a learning model. Like us, \cite{nn_1974} model legal decision making as a nearest neighbor classifier; they also consider in more detail how to represent case fact patterns as points in an instance space. The work of \cite{hartline_et_al} is perhaps most similar to ours in this vein: they also provide a machine learning perspective on the common law system and model strategic agents within this system (although these agents are not strategic litigators in the sense of strategic impact litigation we are focused on). As in our setting, \cite{hartline_et_al} models courts as deciding cases brought by strategic agents according to some decision rule, but unlike in our setting, this rule is unknown. The agents in \cite{hartline_et_al} are utility maximizers. They choose to bring a case or settle out of court depending on which has higher expected utility, taking account of the cost of litigation, utility of various outcomes, and the agents' posterior over the decision rule given past cases. The paper explores whether the system as a whole converges to the unknown decision rule. Our focus on strategic impact litigation requires a different model. First, we consider a hierarchical system with a higher court establishing precedent for a lower court to follow. Second, the agent's goal is to affect the lower court's decision rule for future cases. The utility associated with the individual case is not relevant.

\section{Strategic litigation in one dimension}
Despite its simplicity, the 1-Dimensional Nearest Neighbor setting is already quite rich. We begin by highlighting some interesting structural observations exhibited by our model, even when the litigator is able to bring any fact pattern in the instance space  as a case. Then we give a dynamic programming algorithm for optimally selecting cases from a finite set of fact patterns.

We adopt a geometric view that simplifies the exposition throughout this section. Recall that the strategic litigator's with goal function $g$ wants to minimize $\mathds{E}_{x \sim \mathcal{D}}[\mathds{1} [g(x) \neq f(x)]]$, where $f\gets \mathcal{A}(S_h \cup S_l)$ is the classifier applied by the lower court and $\mathcal{D}$ is some distribution of case fact patterns that the strategic litigator cares about. In 1 dimension, we can view the probability mass under $\mathcal{D}$ as a measure of distance. By doing so, the function the strategic litigator wants to minimize is the total length of all intervals along which $f$ and $g$ disagree. We call this quantity the \textit{discrepancy}. 

\subsection{Structural Observations} 
In these examples, we assume the instance space is the real interval $[0,1]$ for simplicity, and we assume the strategic litigator is able to bring any fact pattern in the instance space as a case. Without loss of generality, we assume the output of Nearest Neighbor when there is no historical data is the all-positive function. Recall that $f^*$ is the labeling function of the high court, $f$ is the output of $\mathcal{A}$ (Nearest Neighbor) on any historical points labeled by $f^*$, and $g$ is the goal function of the strategic litigator. 

\textbf{Observation 1: It can be arbitrarily helpful for the strategic litigator to bring a case where the high court will rule against them.}

\begin{figure}[h!]
    \centering
    \ifdim\columnwidth<\textwidth
        \includegraphics[width=0.9\columnwidth]{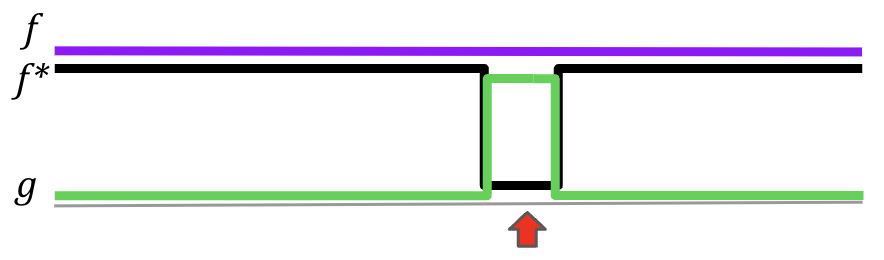}
    \else
        \includegraphics[width=0.6\textwidth]{ex_disagree.png}
    \fi
    \caption{\label{fig:obs1} An example illustrating Observation 1: choosing a losing case (indicated by the red arrow) brings the classifier produced by Nearest Neighbor much closer to the strategic litigator’s
goal function.}
\end{figure}
\vspace{-1em} 
    
For example, in Figure \ref{fig:obs1}, $f^*$ and $g$ disagree on every point, so no matter which case the strategic litigator brings, the high court will rule against them. However, since $f$ is all-positive and $g$ is mostly negative, it benefits the strategic litigator to bring a case which will be classified negatively (such as the point indicated by the red arrow). This will cause Nearest Neighbor to output the all-negative function, which is much closer to $g$ than the current all-positive $f$. Choosing this point could have arbitrarily high benefit as the positive region of $g$ (and the corresponding negative region of $f^*$) makes up a smaller and smaller fraction of the overall interval.

\textbf{Observation 2: Choosing a case myopically can cause unnecessary, irreversible error.}

\begin{figure}[h!]
    \centering
    \ifdim\columnwidth<\textwidth
        \includegraphics[width=0.95\columnwidth]{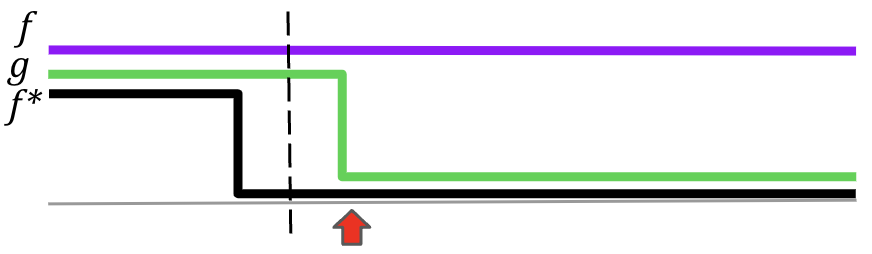}
    \else
        \includegraphics[width=0.6\textwidth]{ex_myopic_extra_error.png}
    \fi
    \caption{\label{fig3} An example illustrating Observation 2: greedily choosing a currently-optimal case
(indicated by the red arrow) causes irreversible error that could have been avoided by planning
ahead.}
\end{figure}
\vspace{-1em} 

One natural question is whether a strategic litigator can use a simple greedy strategy, choosing cases one at a time by selecting the single most helpful case at each step. Figure \ref{fig3} shows that such a strategy can create unnecessary, irreversible error. 
In particular, since the current $f$ is all-positive, bringing any positive point does not alter $f$, but bringing any negative point will be bring $f$ closer to $g$ since it will flip $f$ to the all-negative function. Myopically, then, the optimal choice is any negative point. Suppose the strategic litigator even tie-breaks among these options in a smart way by choosing one of the negative points which $g$ would also label negatively. The strategic litigator might then choose the negative point indicated by the arrow. Choosing this point would create irreversible error: no matter what positive points they bring in the future, the strategic litigator can no longer bring the boundary of $f$ further to the right than the dashed line. This error is unnecessary as it would have been possible to exactly achieve $g$ (by bringing any positive point to the left of $g$'s boundary, and a negative point an equal distance to the right of $g$'s boundary).

\textbf{Observation 3:  1-step lookahead is also not enough to be optimal.}

\begin{figure}[h!]
    \centering
    \ifdim\columnwidth<\textwidth
        \includegraphics[width=\columnwidth]{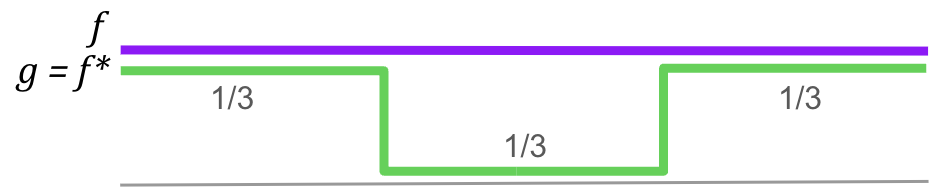}
    \else
        \includegraphics[width=0.6\textwidth]{ex_2_step.png}
    \fi
    \caption{\label{fig4} An example illustrating Observation 3: greedily choosing a currently-optimal pair of cases cannot improve the strategic litigator’s error, even though the strategic litigator’s goal function can be achieved.}
\end{figure}

One might observe that choosing the best pair of points would be sufficient to avoid the problem in the previous example. Since 2 points define a boundary under Nearest Neighbor, it's natural to wonder whether choosing the best pair of points might be enough to overcome the pitfalls of myopia in general. However, even looking ahead to choose the best pair of points, the strategic litigator can fail to minimize their error. For example in Figure \ref{fig4}, no pair of points can improve the strategic litigator's error, so a strategic litigator who chooses pairs of points myopically will not choose any, failing to achieve $g$ even though it is possible (since $g=f^*$). If the strategic litigator brings 2 positive points, $f$ does not change. If the strategic litigator brings 2 negative points, $f$ will become the all-negative function, making the strategic litigator worse off. If the strategic litigator brings a positive and a negative point, $f$ will become some function with a single boundary between a negative and positive region. At best, the strategic litigator will improve with respect to the entire negative third of $g$ by making this region of $f$ negative, but $f$ must then be negative throughout one of the neighboring positive thirds of $g$, so the total discrepancy between $f$ and $g$ does not change.

\textbf{Observation 4: The strategic litigator might need to make things worse in the short term in order to make things better in the longer term.}

\begin{figure}[h!]
    \centering
    \ifdim\columnwidth<\textwidth
        \includegraphics[width=0.9\columnwidth]{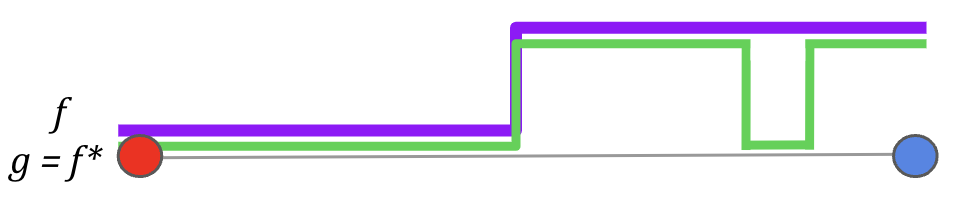}
    \else
        \includegraphics[width=0.6\textwidth]{ex_worse_before_better.png}
    \fi
    \caption{\label{fig5} An example illustrating Observation 4: bringing any individual case increases the
strategic litigator’s error, even though the strategic litigator’s goal function can be achieved by
bringing several cases. The red and blue dots represent historical negative and positive points,
respectively.}
\end{figure}
\vspace{-1em} 

In the example shown in Figure \ref{fig5}, bringing any individual case makes the strategic litigator worse off (it increases the discrepancy between $f$ and $g$), but bringing several cases benefit the strategic litigator (by making $f$ equal $g$). Bringing any negative point will shift the boundary of $f$ to the right, increasing the discrepancy between $f$ and $g$; any positive point will shift the boundary of $f$ to the left, increasing the discrepancy between $f$ and $g$. However, since $g=f^*$, the strategic litigator can make $f$ match $g$ exactly by bringing several points. Note that this also highlights the pitfalls of a myopic strategy, since a myopic strategic litigator would not bring any points at all, failing to achieve $g$ even though it is possible.

\subsection{Optimally selecting cases in one dimension} 
This section presents a dynamic programming algorithm for optimally selecting a subset of cases $X$ from a finite pool of fact patterns $P$ for the 1-Dimensional Nearest Neighbor setting, given labeled historical precedent $S_h$, goal function $g$, and distribution $\mathcal{D}$ over case fact patterns. The high level intuition is as follows: when choosing a set of points one by one from left to right, under Nearest Neighbor the error of the solution after adding a new point $p$ is unchanged from the previous solution over the region from the left boundary of the domain to the midpoint $m$ between $p$ and the closest point to its left. To the right of $m$, the error is now determined only by $p$ and any relevant historical precedent (points from $S_h$ to the right of $p$). This allows us to build an error-minimizing solution iteratively. 

\begin{theorem} \label{thm:alg_correct}
    In the 1-Dimensional Nearest Neighbor setting, for any goal function $g$, set of case fact patterns $P$, distribution $\mathcal{D}$ over $\mathds{R}$, set of historical precedent $S_h$, 
    a strategic litigator can select a subset of case fact patterns $\{x_1, x_2, ...\} \subseteq P$  such that $f \gets \mathcal{A}(S_h \bigcup \allowbreak \{(x_1, f^*(x_1)), \dots\})$ minimizes $\mathds{E}_{x \sim \mathcal{D}}[\mathds{1} [g(x) \neq f(x)]]$ over all possible subsets of $P$ in $O(|P|^2)$ calls to a nearest-neighbor oracle and a discrepancy oracle which computes the error of the current function with respect to $g$.
\end{theorem}

The proof of Theorem \ref{thm:alg_correct} can be found in \Cref{appendix_alg_1}. 

\begin{algorithm*}[]
    \small    
    \caption{Select a subset of case fact patterns $\{x_1, x_2, ...\} \subseteq P$  such that $f \gets \mathcal{A}(S_h \bigcup \{(x_1, f^*(x_1)), ...\})$ minimizes $\mathds{E}_{x \sim \mathcal{D}}[\mathds{1} [g(x) \neq f(x)]]$.}\label{dp_1d}
    \begin{algorithmic}
         \STATE \textbf{Inputs:}\\ \begin{itemize}
             \item $f^*$: the high court's labeling function.
             \item  $g$: the strategic litigator's goal function.
             \item  $\mathcal{D}$: the strategic litigator's distribution of case fact patterns of interest.
             \item $S_h$: historical labeled data.
             \item $P$: collection of available case fact pattern points.
         \end{itemize}
        \STATE \textbf{Output:} A subset of case fact patterns $X \subseteq P$ which minimizes the strategic litigator's expected error with respect to $g$ and $\mathcal{D}$. \\
        \STATE \textbf{Utilities:} \begin{itemize}
            \item Let \textbf{nearest\_neighbor}(\textit{interval}, \textit{S}) compute the standard nearest neighbor classifier over the interval \textit{interval} given points in $S$. \\ 
            \item Let \textbf{discrepancy}(\textit{f}, \textit{g}, \textit{interval}, $\mathcal{D}$) compute the sum of the lengths of intervals within \textit{interval} throughout which $f(x) \neq g(x)$, where units of distance are probability mass under $\mathcal{D}$.
        \end{itemize}
    \end{algorithmic}        
    \begin{algorithmic}[1]  
        \STATE Rescale coordinates so all inputs (points, functions, $\mathcal{D}$) are with respect to the interval $[0,1]$
        \STATE Sort $P$ from leftmost to rightmost point
        \STATE Construct a DP table $T$ which is a list of length $p$, where $p$ is the number of points in $P$. $T[j]$ will store the strategic litigator's minimum total error if $P[j]$ is the rightmost point selected.
        \STATE $f_0 =$ \textbf{nearest$\_$neighbor}([0,1], $S_h \bigcup \{(P[0], f^*(P[0]))\}$) \textit{output of Nearest Neighbor when choosing only $P[0]$}
        \STATE $T[0] = $ \textbf{discrepancy}($f_0$, $g$, [0,1], $\mathcal{D}$) \textit{error when choosing only $P[0]$}
        \FOR{$j = 1$ to $p-1$}
             \STATE $f_j =$ \textbf{nearest$\_$neighbor}([0,1], $S_h \bigcup \{(P[j], f^*(P[j]))\}$) \textit{output of Nearest Neighbor when choosing only $P[j]$}
             \STATE $err\_only\_j =$ \textbf{discrepancy}($f_j$, $g$, [0,1], $\mathcal{D}$) \textit{error when choosing only $P[j]$}
             \STATE $err\_j\_after\_i = \min_{i \in [j-1]}$ \textbf{error}($i$, $j$) \textit{calculate the lowest possible error when choosing point $j$ if $i$ is the closest chosen point to its left}
             \STATE $T[j] = \min\{err\_only\_j, err\_j\_after\_i\}$
        \ENDFOR
        \STATE $min\_err\_with\_points = \min_j T[j]$
        \STATE $f_\emptyset = $\textbf{nearest$\_$neighbor}([0,1], $S_h$) \textit{output of Nearest Neighbor when choosing no points}
        \STATE $err\_no\_points =$ \textbf{discrepancy}($f_\emptyset$, $g$, [0,1], $\mathcal{D}$) \textit{error when choosing no points}
        \IF{$err\_no\_points \leq min\_err\_with\_points$} 
            \RETURN $\emptyset$ 
        \ELSE
        \RETURN the set of points corresponding to $\min_j T[j]$. We can find these by retracing the entries of $T$ we used in the minimum cost solution, and adding each of the corresponding points to our output set. 
        \ENDIF
    \end{algorithmic} 
    \begin{algorithmic}
        \STATE
        \STATE \textit{Local helper function which computes the minimum error when the $ith$ and $j$th points are the two rightmost selected points (the $j$th point is to the right of the $i$th point)}
        \STATE \textbf{procedure} \textbf{error}($i$, $j$):  
    \end{algorithmic}
\begin{algorithmic}[1]
        \STATE $f_i$ = \textbf{nearest$\_$neighbor}([P[i],1], $S_h\bigcup\{(P[i], f^*(P[i]))\}$) \textit{output of Nearest Neighbor from $P[i]$ onward when $P[i]$ is the rightmost selected point}
        \STATE $err\_after\_i\_no\_j = $ \textbf{discrepancy}($f_i$, $g$, [P[i],1], $\mathcal{D}$) \textit{error to the right of $P[i]$ without j}
         \STATE $f_{ij}$ = \textbf{nearest$\_$neighbor}([P[i],1], $S_h \allowbreak \bigcup\ \allowbreak \{ \allowbreak (P[i], f^*(P[i])), \allowbreak (P[j], f^*(P[j])\}$) \allowbreak \textit{output of Nearest Neighbor from $P[i]$ onward when we choose $P[j]$ next after $P[i]$}
        \STATE $err\_after\_i\_with\_j = $\textbf{discrepancy}($f_{ij}$, $g$, [P[i],1], $\mathcal{D}$) \textit{error to the right of $P[i]$ when we choose $P[j]$ next}
        \RETURN $T[i] - err\_after\_i\_no\_j + err\_after\_i\_with\_j$ ~ \textit{minimum overall error when choosing $P[j]$ immediately after $P[i]$}
    \end{algorithmic}
\end{algorithm*}

\remark{If the strategic litigator also wants to limit the number of cases they bring, \Cref{dp_1d} can be easily extended to find the optimal set of at most $k$ case fact patterns. We can simply extend the DP table to have $k$ rows of length $p$ instead of a single one. $T[j][i]$ would then store the minimum error achievable with at most $j$ points when the $i$th point is the rightmost point selected, and the table could be filled row by row, from left to right.}

\section{Strategic litigation in $d$ dimensions}

This section considers the $d$-Dimensional Linear Separator setting. We first characterize which goal functions $g$ are \emph{achievable} in principle, assuming the strategic litigator is free to bring any cases of his choosing. We show that all achievable $g$ are achievable using only 2 cases. We then discuss two more challenging settings: when $g$ is not achievable, and when the strategic litigator is restricted to bringing cases from a finite pool $P$.

\subsection{Which goal functions are achievable?}
 In an ideal case, the strategic litigator would like to select a set of case fact patterns such that function output by $\mathcal{A}$ (the lower court's learning algorithm) is exactly $g$. When can the strategic litigator force $\mathcal{A}$ to produce exactly $g$, assuming they can select any case fact pattern in the space? We answer this question in Theorem \ref{thm:acheivable} below.

\begin{definition}\label{def:achievable}[Achievable function]
We call a function $g$ achievable if there exists some set $\{x_1, ..., x_m\}$ of case fact patterns such that
$g \gets \mathcal{A}\big(\{(x_1, f^*(x_1)), \allowbreak ..., \allowbreak (x_m, f^*(x_m))\} \cup S_h\big)$,
where $f^*$ is the labeling function of the high court, $\mathcal{A}$ is the learning algorithm of the low court, and $S_h$ is the set of historical data.
\end{definition}

\begin{theorem}\label{thm:acheivable} Let $S_h$ be the set of historical cases labeled by $f^*$. Let $g$ be some linear function and let $\theta$ be the angle between the normal vectors of $g$ and $f^*$ (call these $w_g$ and $w_{f^*}$ respectively) in degrees. Assume the strategic litigator can bring any case fact pattern in the space. Then, $g$ is achievable if and only if one of the following conditions hold: 
\begin{itemize}
    \item $0 < \theta < 90$, there is no historical data in the disagreement region of $f^*$ and $g$ (i.e. $\{(x, y) \in S_h: f^*(x) \neq g(x)\} = \emptyset$), and there are no points directly on the boundary of $g$.
    \item $\theta = 0$ (i.e. $f^*$ and $g$ are parallel and positive in the same direction), and there is no historical data less than distance $\delta$ from $g$, where $\delta$ is the maximum of the distances from $g$ to the positive region defined by $f^*$ and to the negative region defined by $f^*$. 
\end{itemize}
\noindent
Moreover, if g is achievable, it is achievable using 2 cases.
\end{theorem}

    \begin{figure}[]
        \centering
        \includegraphics[scale=.4]{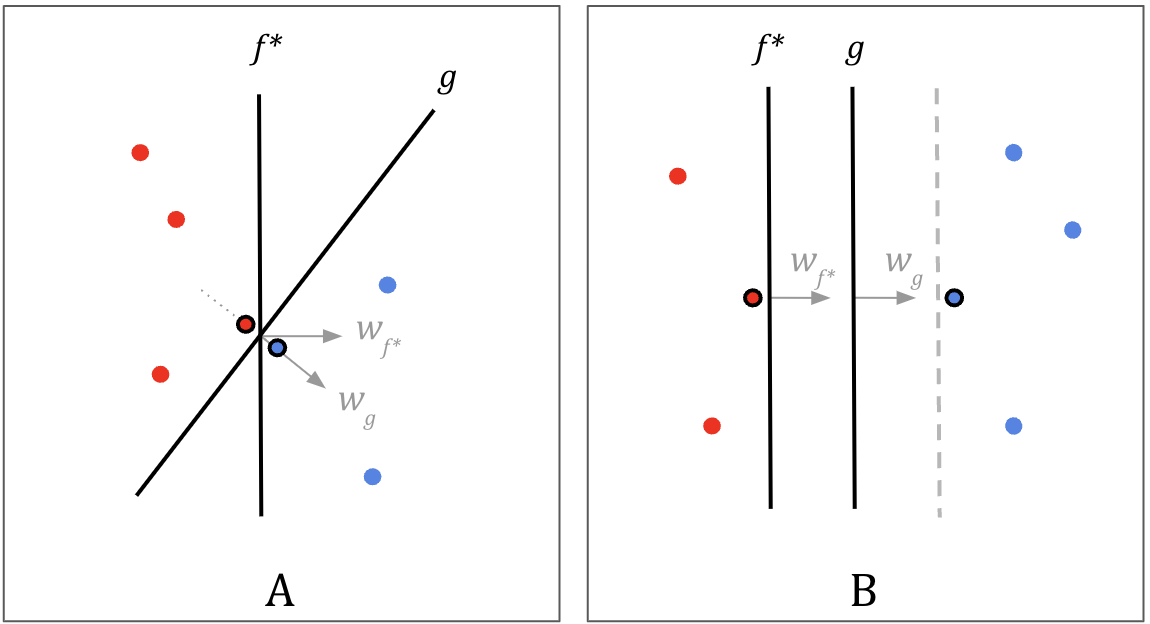}
        \caption{Frame A illustrates the first case from Theorem \ref{thm:acheivable}. The two points with bolded border lying on $w_g$ and $-w_g$ will become the support vectors, producing $g$ (the other points illustrate historical data). Frame B illustrates the second case; again, the two points with bolded border will become the support vectors, producing $g$. The dashed line marks distance $\delta$ away from $g$ on the right. \label{fig:thm}}
    \end{figure}

\begin{proof}
    First, we show that SVM cannot produce any $g$ whose normal vector $w_g$ has angle $\geq$ 90 degrees to the normal vector of $f^*$. SVM produces some linear classifier whose normal vector is a linear combination of the data points, where the weights of the data points have the same sign as the their label. This implies that the normal vector of any function SVM can produce using data labeled by $f^*$ has a positive inner product with $w_{f^*}$. Two vectors with a positive inner product have angle $< 90$ degrees.\\
    The rest of the negative results are easy to see: If any point $p$ lies directly on $g$, we cannot teach $g$ itself to SVM (although we can get arbitrarily close). Suppose WLOG $p$ is positive; since SVM produces a maximum-margin classifier, the classifier it produces can only ever be as close to $p$ as to the closest negative point the strategic litigator can bring; this distance cannot be 0 when $g$ is a linear separator. Furthermore, SVM cannot produce any $g$ which is inconsistent with any historical data, because SVM is guaranteed to find a consistent function in $\mathcal{H}$ if one exists, and $f^*$ is such a function. Finally, suppose $f^*$ and $g$ are parallel and positive in the same direction, and there is historical data within distance $\delta$ of $g$, where $\delta$ is the maximum of the distances from $g$ to the positive region defined by $f^*$ and to the negative region defined by $f^*$. If any point is less than $\delta$ away from $g$, WLOG let it be positive. The closest negative point must be at least $\delta$ away from $g$ by definition of $\delta$, so SVM cannot produce $g$ as it will produce some classifier with equal minimum margin on both sides. 

    It remains to show that we can produce any $g$ which does not violate the above constraints, and do so by bringing just 2 points. In the special case where $f^*=g$ and there is no historical data, the results of \cite{teaching_linear_learner} apply: they showed that the number of points required to teach a particular separator to SVM is 2. In general, if $f^*$ and $g$ are parallel in the same direction, we can choose one positive and one negative point equidistant from $g$ along the normal vector of $g$, at least as close to $g$ as any historical data (any pair of equidistant points at least $\delta$ away on either side of $g$ will suffice; we can find such points since we assume that there are no historical points less than $\delta$ away from $g$). If $f^*$ and $g$ intersect, without loss of generality let the intersection point be the origin. Then it suffices to choose one positive and one negative point along the normal vector of $g$, arbitrarily closer to $g$ than the distance of any historical point to $g$. In either case, since the two points we choose have different labels, and are equidistant from $g$ on either side of $g$ along a line perpendicular to $g$, $g$ is the maximum margin separator with respect to only this pair of points. The two points we choose must be closer to each other than any other pair of points from different classes are to each other, so they will become the only support vectors (the only points which determine the output of SVM). Therefore, SVM will produce $g$ with the addition of these two points.  
\end{proof}

\subsection{Approximating non-achievable goals functions}\label{sec:not_achievable}

In case the strategic litigator's goal function $g$ is not achievable (see Definition \ref{def:achievable}), the strategic litigator should choose an achievable function $g'$ which minimizes the strategic litigator's classification error (the error according to $g$, with respect to whatever distribution $\mathcal{D}$ of case fact patters the strategic litigator cares about). Once the strategic litigator has selected this achievable proxy goal $g'$, our results from Theorem \ref{thm:acheivable} apply, and the strategic litigator can achieve $g'$ by bringing two case fact patterns as described in the proof of Theorem \ref{thm:acheivable} above. 

The remaining task in this case is therefore to find an achievable $g'$ which minimizes $\mathds{E}_{x \sim \mathcal{D}}[\mathds{1} [g(x) \neq g'(x)]]$. The strategic litigator can find $g'$ using the following simple brute force method. Draw a sample $S$ of points from $\mathcal{D}$, and enumerate all achievable (according to the constraints specified in Theorem \ref{thm:acheivable}) SVM classifiers which correspond to distinct labelings of points in $S$. Choose as $g'$ a classifier in this set which minimizes the number of misclassified points in $S$. Note that by the Sauer-Shelah Lemma \cite{SAUER1972145}, the number of effective classifiers is polynomial in $|S|$. When the size of the sample is large enough, this approximates the best achievable classifier with respect to error on $\mathcal{D}$. Using a standard uniform convergence argument and the fact that the VC dimension of a $d$-dimensional linear classifier is $d+1$, $|S| = \widetilde{O}(\frac{d}{\epsilon^2})$ is enough to guarantee that the error of any classifier $g'$ on $S$ is $\epsilon$-close to the true error over $D$ for some $\epsilon$.

\subsection{Optimally selecting cases from a finite pool}
Suppose the strategic litigator can only select case fact patterns from a limited pool $P$. The output function of SVM is determined entirely by its support vectors (points on the margin). In $d$ dimensions, there can be at most $d+1$ linearly independent support vectors, so any function achievable using \textit{any} subset of $P$ can also be achieved using a subset of size at most $d+1$. Trivially, then, the strategic litigator can enumerate the set $F$ of all achievable functions by enumerating the functions induced by  each of the $O(n^{d+1})$ subsets of size $\leq d+1$ in $P$ (for a given subset $\{x_1, x_2, ...\}$, the corresponding separator is $\mathcal{A}(S_h\bigcup\{(x_1, f^*(x_1)), (x_2, f^*(x_2)),... \})$). It follows that if $g \in F$, it is achievable and the strategic litigator has already found the corresponding subset of $P$ needed to achieve it. Otherwise, if $g \notin F$, $g$ is not achievable. In that case, just as described above in Section \ref{sec:not_achievable}, the strategic litigator can choose a $g' \in F$ which approximately minimizes their classification error on their distribution of interest ($\mathcal{D}$) by choosing $g'$ which minimizes their error on a sufficiently large sample from $\mathcal{D}$ (the same sample complexity bound holds here). In this case, the strategic litigator has already found the corresponding subset of $P$ to teach $g'$. 

\section{Extension: Overturning Precedent}\label{sec:extensions}
We also study the case where there is some historical precedent which does not agree with the high court's current reasoning $f^*$, and model a process for overturning precedent. A change in the high court's labeling function could correspond to a change in justices, or simply to a change in the existing justices' views.

Let $S_h$ be the set of all historical cases and their labels. These cases could have been labeled by $f^*$ or by earlier high court function(s).

We assume the current high court has some limited power to throw out precedent which does not agree with their current reasoning represented by $f^*$ (e.g, points whose label does not match the label $f^*$ would assign). If the high court throws out some precedent (a set of labeled data points), those data points are removed from the set of precedent $S_h$. 

We assume the high court only throws out precedent when deciding a new case (e.g. classifying a new fact pattern $x$), and then only if a certain condition holds: the high court only throws out any precedent if when the newly labeled point ($x$, $f^*(x)$) is added to the data set $S_h$, there is no longer any function in the hypothesis class $\mathcal{H}$ which is consistent with all the data in $S_h$. 

When the high court does throw out precedent, we assume they only remove points whose label disagrees with $f^*$. Moreover, we assume they do not gratuitously throw out precedent: we assume they throw out some minimal set such that there exists a function in $\mathcal{H}$ which is consistent with the remaining data. 

\textbf{Example:} Figure \ref{fig:overturning_precedent} illustrates our model of overturning precedent.

    \begin{figure}[h]
        \centering
        \includegraphics[scale=.4]{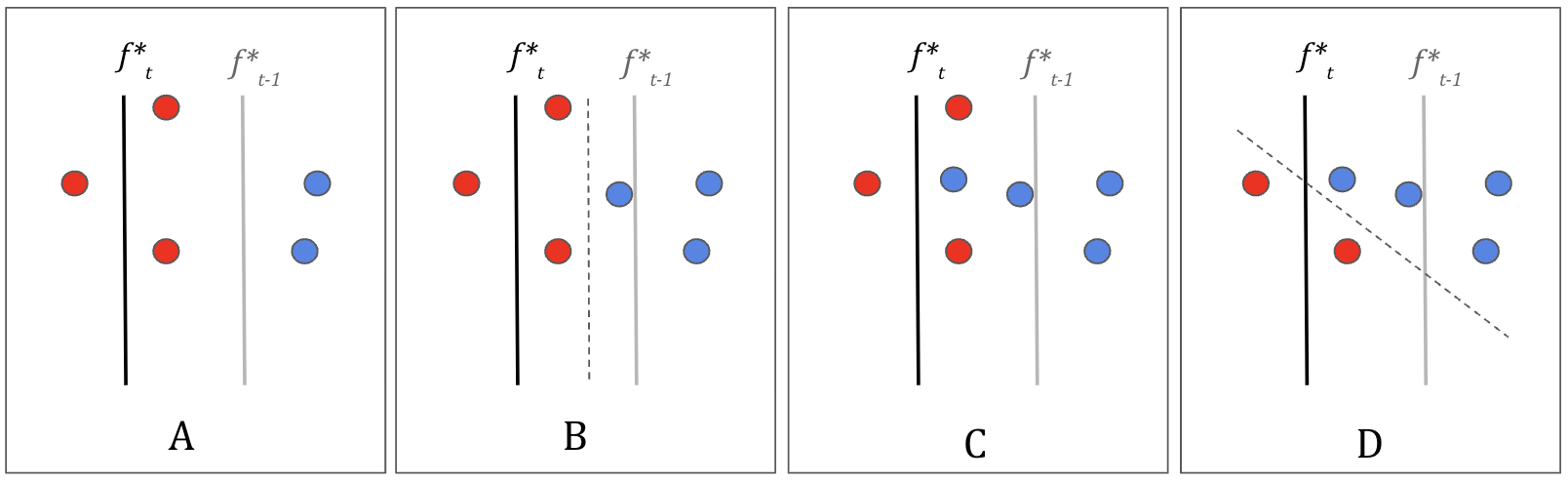}
        \caption{Frame A shows historical data labeled by a previous high court function $f^*_{t-1}$, when a new high court function $f_t^*$ arrives. Frame B shows a new point arriving which is classified positively by $f_t^*$. No precedent is overturned as there still exists a linear separator consistent with the data (e.g., the dashed line). Frame C shows another point arriving which is classified positively by $f^*$. Now there is no longer a linear separator consistent with the data, so the high court must remove some subset of the points whose labels it disagrees with (the two negative points to the right of $f^*_t$). Frame D illustrates if the high court chooses (arbitrarily) to remove the top negative point. It does not remove both because it only removes a minimal set such that there is a linear separator consistent with the remaining data (e.g., the dashed line). }\label{fig:overturning_precedent}
    \end{figure}

 We now return to considering whether the strategic litigator can achieve their goals in this extended setting. Notice that the strategic litigator may now need to strategize about how to force the high court to overturn precedent which disagrees with $g$. We assume the strategic litigator knows the high court's model for overturning precedent as described above, but does not know which minimal set of precedent the high court will choose to remove if there are multiple possibilities.

\subsection{Results in 1-Dimensional Nearest Neighbor Setting}
This extension is trivial in our 1-Dimensional Nearest Neighbor setting. Under a Nearest Neighbor classifier and our model for overturning precedent, the only way to get the high court to remove a historical point is to bring that exact point again and have it labeled the opposite way. If any existing points whose label disagrees with $f^*$ are in the pool of case fact patterns the strategic litigator can choose from, Algorithm \ref{dp_1d} is already set up to evaluate whether the strategic litigator should choose them -- only a minor modification is needed to make sure we compute our Nearest Neighbor error with respect to the new label when a duplicate point is chosen.

\subsection{Results in $d$-Dimensional Linear Separator Setting}
Now we consider this extension in our $d$-dimensional linear separator setting, assuming the strategic litigator can choose any case fact pattern in the space. If $g$ is achievable according to the constraints specified in Theorem \ref{thm:acheivable} when taking into account only historical data labeled by $f^*$,  $g$ is still achievable with the addition of historical data whose labels disagree with $f^*$, so long as these new points also conform to the achievability constraints described in Theorem \ref{thm:acheivable} (e.g., there is no data within distance $\delta$ of $g$ in the case where $g$ is parallel to $f^*$, and no data in the disagreement region of $f^*$ and $g$ in the case where $f^*$ and $g$ intersect).\footnote{In the case where $f^*$ and $g$ intersect, this restriction is actually not necessary: any data in the disagreement region of $f^*$ and $g$ whose label disagrees with $f^*$ must be labeled according to $g$. Additional points whose labels agree with $g$ have no affect on our earlier results in this case. However, we will focus on the more restricted case for simplicity in directly applying our results from Theorem \ref{thm:acheivable}.}

Trivially, the strategic litigator can achieve $g$ in this case by bringing up all the historical points whose label disagrees with both $g$ and $f^*$. They will all be relabeled by $f^*$ and therefore their labels will also agree with $g$, returning the strategic litigator to precisely the achievable setting described in Theorem \ref{thm:acheivable}. The strategic litigator can then bring 2 points as described in the proof of Theorem \ref{thm:acheivable} to achieve $g$. However, this could potentially require bringing many points. The strategic litigator can also achieve $g$ using at most $2d+1$ points, as stated in Theorem \ref{thm:overturn_precedent}.

 \begin{theorem} \label{thm:overturn_precedent}
    In our $d$-dimensional linear separator setting, any linear classifier $g$ is achievable if it satisfies the conditions in Theorem~\ref{thm:acheivable}. Moreover, if $g$ is achievable, it is achievable by bringing at most $2d+1$ points.
\end{theorem}

\begin{proof}
The proof idea is to bring a set $N$ of points satisfying that
\begin{itemize}
    \item any linear classifier consistent with $N$ will label the existing points in the same way as $f^*$;
    \item and similar to Theorem~\ref{thm:acheivable}, $N$ contains two points to teach $g$.
\end{itemize}
Let $S_{dis}$ denote the set of existing points that disagree with $f^*$ and $g$.
For each linear classifier $w^\top x +b\geq 0$, we can represent it by $(w,b)$. Then we can denote the set of all linear separators by $\{(w,b)| \|w\|_2=1\}$ (here we restrict $\|w\|_2=1$ since the scaling doesn't matter).
We first prove the statement for the first condition, where \textbf{the angle between $f^*$ and $g$ satisfies $0<\theta<90$}.

\textbf{Proof for  the case of $d=2$.} Let's start with $2$-dimensional case. We pick a coordinate system such that $f^*$ is represented as $e_1^\top x \geq 0$ and $g$ is represented as $w_g^\top x\geq 0$ with $w_{g,1} >0,w_{g,2}\geq 0$, where
$e_1,e_2,\ldots$ denote the standard basis of this coordinate system. 
Let $w_g^\perp = \frac{e_1 - (e_1^\top w_g) w_g}{\|e_1 - (e_1^\top w_g) w_g\|_2}$ denote the normalized vector perpendicular to $w_g$. 
When there are no existing disagreeing points lie on the boundary of $f^*$, we bring $4$ points $N = \{(\epsilon w_g, +1), (-\epsilon w_g, -1), (\epsilon w_g+\alpha w_g^\perp, +1), (\alpha e_2, +1)\}$ for some large $\alpha$ and small $\epsilon$. 
For any negative disagreeing point $(x,-1)$ (satisfying $f^*(x)=g(x)=+1$ and not on the boundary of $g$), the feature $x$ can be decomposed as $x = \alpha_1 w_g^\perp +\alpha_2 e_2$ for some $\alpha_1, \alpha_2 > 0$ (since $e_1^\top x = \alpha_1 \frac{1 - (e_1^\top w_g) w_{g,1}}{\|e_1 - (e_1^\top w_g) w_g\|_2}>0$ and $w_g^\top x = \alpha_2 w_{g,2}$). 
Hence, when $\epsilon \rightarrow 0$ and $\alpha\rightarrow \infty$, $x$ must lie inside the convex hull of $\epsilon w_g, \epsilon w_g+\alpha w_g^\perp, \alpha e_2$ with some margin (i.e., $x$ is not on the boundary of this convex hull). 
Any linear separator consistent with $N$ will label this convex hull as positive and thus $x$ as positive by some positive margin (which depends on $x$ only) and thus $(x,-1)$ will be deleted.
Similarly, for any positive disagreeing point $(x,+1)$ (satisfying $f^*(x) = g(x) = -1$), when $\epsilon \rightarrow 0$ and $\alpha\rightarrow \infty$, $-x$ must lie in the convex hull of $\epsilon w_g, \epsilon w_g+\alpha w_g^\perp, \alpha e_2$. Hence any consistent linear separator will label $-x$ as positive by some positive margin. Thus $2\epsilon w_g -x$ will also be labeled as positive when $\epsilon$ is small enough.

If there are negative disagreeing points lying on the boundary of $f^*$, pick the one closest to the origin and denote its distance by $c$.
Then we add $(\frac{c}{2} e_2,+1)$ to $N$. Then all existing negative points in $S_{dis}$ points on the boundary will lie in the convex hall of $\frac{c}{2} e_2$ and $\alpha e_2$. Thus, these points will be deleted. 
Hence, in the $2$-d case, we can delete all disagreeing points by bringing $5$ points. The two points $(\epsilon w_g,+1)$ and $(-\epsilon w_g,-1)$ will teach $g$ as we showed in the proof of Theorem~\ref{thm:acheivable}.

\textbf{Proof for the high dimensional case.} Now we reduce the high dimensional case to the $2$-d case. Again we pick a coordinate system such that $f^*$ is $e_1^\top x \geq 0$ and $g$ is $w^\top x\geq 0$ with $w_1 >0,w_2\geq 0, w_i=0, \forall i\geq 3$.
Similar to the $2$-d case, we add $(\epsilon w_g,+1), (-\epsilon w_g,-1)$ to $N$. Then any linear classifier consistent with these two points must have $b\in [-\epsilon,\epsilon]$. Otherwise it will label these two points in the same way and be inconsistent with $N$. 
Then we add points $(\epsilon w_g + \infty \cdot e_3, +1)$ and  $(\epsilon w_g - \infty \cdot e_3, +1)$ to $N$. Any linear separator $w$ with $w_3 \neq 0$ will label these two points in two different ways. Hence, any linear separator consistent with $N$ will satisfy $w_3=0$. Similarly, we can also add such pair of points for other dimensions. Hence, by adding $2(d-2)$ points, we reduce the high dimensional case to the $2$-dimensional case.

\textbf{Under the second condition where $\theta =0$ and there are no historical data with in distance $\delta$ of $g$.} Pick the coordinate system such that $g$ is $e_1^\top x \geq 0$ and $f^*$ is $e_1^\top x \geq -\delta$. Let $e_1,e_2,\ldots$ denote the standard basis of this coordinate system. Then we can add $N= \{(-(\delta+\epsilon)e_1, -1),((\delta+\epsilon)e_1, +1),((\delta+\epsilon)e_1 +\infty\cdot e_2, +1), ((\delta+\epsilon)e_1 -\infty\cdot e_2, +1),\ldots, ((\delta+\epsilon)e_1 +\infty\cdot e_d, +1),((\delta+\epsilon)e_1 -\infty\cdot e_d, +1)\}$. By adding $N$, we can guarantee that any consistent linear separator will label the entire hyperplane $e_1^\top x= \delta +\epsilon$ as positive.  For any negative disagreeing point $(x,-1)$ with $x_1>\delta$, when $\epsilon\rightarrow 0$, the segment with ending points $-(\delta+\epsilon)e_1$ and $x$ will intersect with the hyperplane $e_1^\top x= \delta +\epsilon$. Then any consistent linear separator will not label $x$ as negative; otherwise both $-(\delta+\epsilon)e_1$ and $x$ are labeled as negative and thus, the intersecting point is also labeled as negative, which is a contradiction.
\end{proof} 

\section{Discussion}

This work is a preliminary exploration of strategic impact litigation from a machine learning theory perspective. We hope it can serve as a jumping off point for further study in this direction. Our model of the legal system is highly stylized, and there are numerous ways to enrich it. We will mention a few potentially interesting variations here.

One direction could be to model uncertainty about decisions in the legal process. Throughout this paper, we assumed the strategic litigator knows the function $f^*$ which the high court uses to decide cases. We could alternatively consider a setting where, while the high court still makes decisions according to some $f^*$, the strategic litigator only knows a set of, or distribution over, possible functions the high court could use.  This would bring in aspects of (agnostic) active learning, for which a number of interesting techniques have been developed \cite{balcan2006agnostic, balcan2015active, awasthi2017power}. To think about uncertainty another way, we could also model the courts as making noisy decisions; this could capture the fact that a variety of unpredictable factors outside of the case fact pattern itself--from the random selection of potential jury members, to whether the judge is hungry while making the decision--can impact the outcome of a case. 

We could also explore a more complex model of the courts; for example, the high court might provide some explanation of their reasoning about decisions in addition to the decision itself, which could then be used by lower courts to inform the classifier they learn. 

Another interesting direction might be to model additional litigators. In this work, while there might be historical cases the strategic litigator did not bring, the strategic litigator benefits from being the only currently active litigator: the strategic litigator can bring as many strategically selected cases as they like without any cases being brought by others throughout this process. New cases being brought by other agents might render the some of the cases initially selected by the strategic litigator less helpful or potentially even harmful. We could consider what outcomes might be achieved in a setting with multiple strategic litigators with competing interests, or perhaps in a less competitive setting where cases drawn at random from some distribution arrive periodically.


\bibliography{refs}

\appendix
\setcounter{algorithm}{0}                    
\renewcommand{\thealgorithm}{\arabic{algorithm}}  
\renewcommand{\theHalgorithm}{\thesection.\arabic{algorithm}} 

\setcounter{algorithm}{0}

\section{Proof of Theorem \ref{thm:alg_correct}}\label{appendix_alg_1}
Here we'll give a proof of the correctness and runtime of \Cref{dp_1d} (Theorem \ref{thm:alg_correct}). 

\subsection{Proof of Theorem \ref{thm:alg_correct} \label{sec:apdx}}

\begin{proof}
First, we'll prove the correctness of Algorithm \ref{dp_1d}. We will assume the utility functions are correct, so 
    \hbox{\textbf{nearest\_neighbor}(\textit{interval}, \textit{S})}
    correctly computes a standard nearest neighbor classifier over the interval \textit{interval} given points in \textit{S}, and \hbox{\textbf{discrepancy}(\textit{f}, \textit{g}, \textit{interval}, $\mathcal{D}$)}
correctly computes the sum of the lengths of intervals within \textit{interval} throughout which $f(x) \neq g(x)$, where units of distance are probability mass under $\mathcal{D}$. Then, \textbf{discrepancy}(\textit{f}, \textit{g}, \textit{interval}, $\mathcal{D}$) also computes the strategic litigator's error for producing $f$ over the interval \textit{interval}, since this is the total probability mass on points $f$ misclassifies over the interval \textit{interval}. These will be used throughout the rest of the proof.

    Recall that we begin by re-scaling coordinates so all inputs ($f$, $g$, all points, and $\mathcal{D}$) are with respect to the interval $[0,1]$, so throughout the rest of this proof, $[0,1]$ is the entire interval and $[p,1]$ is the entire interval including and to the right of some point $p$. 

    The minimum-error set of points is either empty or contains some points, and we return the set corresponding to the minimum error between these two options (lines 15-18). To show the algorithm is correct, then, it remains to show that the error in each case is correctly computed and corresponds to a valid set of points. 

    When the strategic litigator does not choose any points, the strategic litigator's error is the error corresponding to the output of Nearest Neighbor over the entire interval [0,1] given only the set of historical points $S_h$, so we correctly compute the minimum error in this case. When this case gives the minimum error overall, we correctly return $\emptyset$ (provided we also correctly compute the minimum error for the case when the strategic litigator does choose some points, which we will show next).

    If $T[j]$ correctly stores the minimum error when $P[j]$ is the rightmost point selected, we correctly compute the minimum error in the case where the strategic litigator selects at least one point as $\min_jT[j]$, since some selected point must be the rightmost selected point. We'll prove that $T[j]$ stores the minimum error when $P[j]$ is the rightmost selected point by induction. Recall that we start by sorting the collection of points $P$ from left to right, so when we refer to the $i$th point, it is the $i$th point from the left.
    
    Consider the base case, T[0]. When choosing $P[0]$ as the rightmost point, it must be the only selected point. This corresponds to the error of the output of Nearest Neighbor over the entire interval [0,1] given $P[0]$ labeled by $f^*$ and the rest of the historical points, which is exactly what we compute. Clearly this case corresponds to a valid set of points (\{P[0]\}).

    Next, assuming $T[k]$ correctly stores the minimum error when the $k$th point is the rightmost selected point for some $k \geq 0$, we'll show $T[k+1]$ is also computed correctly. When the $k+1$th point is the rightmost selected point, either it is the only selected point or there is at least one other selected point to its left, so the minimum error overall is the minimum of these two cases.

    The first case corresponds to the error of the output of Nearest Neighbor over the entire interval [0,1] given $P[k+1]$ labeled by $f^*$ and the set of historical points, which is exactly what we compute. Clearly this case corresponds to a valid set of points (\{P[k+1]\}).

    Now consider the second case where we select point $P[k+1]$ and at least one other point $P[i]$ which is the closest selected point to its left. The selection of the $P[k+1]$ does not alter the Nearest Neighbor classifier over the interval [0,P[i]], and only $P[i]$ and points to its right affect the classifier over the interval [P[i], 1]. Therefore, the minimum error achievable in this case is the total minimum error when $P[i]$ is the rightmost point over the entire interval [0,1] \textit{less} the error to the right of $P[i]$ when $P[i]$ is the rightmost selected point, \textit{plus} the error to the right of $P[i]$ when $P[k+1]$ is the rightmost selected point. Since by our inductive hypothesis $T[i]$ correctly stores the total minimum error when $P[i]$ is the rightmost point over the entire interval [0,1], we correctly compute the minimum error in this case (see lines 1-5 of the helper function \textbf{error}). Again since the selection of the $P[k+1]$ does not alter the Nearest Neighbor classifier over the interval [0,P[i]], and only $P[i]$ and points to its right affect the classifier over the interval [P[i], 1], simply adding $P[k+1]$ to the set of points corresponding to the best error when $P[i]$ is the rightmost point is a valid set of points corresponding to the minimum error we store in $T[k+1]$. This concludes the proof of correctness.
    
    That the algorithm uses $O(|P|^2)$ oracle calls is easy to see: the table has $|P|$ entries, and to compute each entry $T[i]$ for $0 \leq i \leq |P|$, we do a constant number of oracle calls to \textbf{nearest\_neighbor} and \textbf{discrepancy} involving every entry to the left of $T[i]$. Although it depends on implementation, we'd expect that each call to \textbf{nearest\_neighbor} would take up to $O(|S_h|)$ time as used here, while the time complexity of each call to \textbf{discrepancy} would be at most linear in the total number of ``alternations'' in both functions being compared (i.e., the number of times each function switches from negative to positive or vice versa). Here, the number of alternations is linear in $|S_h|$ plus the number of alternations in $g$, which is arbitrarily selected by the strategic litigator. 
\end{proof}

\end{document}